%% file: Global_Hypothesis_Generation_for_6D_Object_Pose_Estimation_-_Arxiv.tex
\documentclass[10pt,twocolumn,letterpaper]{article}

\usepackage{cvpr}
\usepackage{times}
\usepackage{epsfig}
\usepackage{graphicx}
\usepackage{amsmath}
\usepackage{amssymb}
\usepackage{amsthm}
\usepackage{booktabs}
\usepackage{multirow}
\usepackage{subfigure}

\usepackage[pagebackref=true,breaklinks=true,letterpaper=true,colorlinks,bookmarks=false]{hyperref}

\cvprfinalcopy 

\def\cvprPaperID{*****} 

\newcommand{\citep}[1]{\cite{#1}}

\newcommand{\BR}{\mathbb{R}}
\newcommand{\BL}{\mathbb{L}}
\newtheorem{proposition}{Proposition}
\newtheorem{definition}{Definition}
\newtheorem{corollary}{Corollary}

\ifcvprfinal\fi
\begin{document}

\title{Global Hypothesis Generation for 6D Object Pose Estimation }

\author{Frank Michel, Alexander Kirillov, Eric Brachmann, Alexander Krull\\ Stefan Gumhold, Bogdan Savchynskyy, Carsten Rother\\
TU Dresden\\
}

\maketitle

\begin{abstract}
This paper addresses the task of estimating the 6D pose of a known 3D object from a single RGB-D image. Most modern approaches solve this task in three steps: i) Compute local features; ii) Generate a pool of pose-hypotheses; iii) Select and refine a pose from the pool. This work focuses on the second step. While all existing approaches generate the hypotheses pool via local reasoning, e.g. RANSAC or Hough-voting, we are the first to show that global reasoning is beneficial at this stage. In particular, we formulate a novel fully-connected Conditional Random Field (CRF) that outputs a very small number of pose-hypotheses. Despite the potential functions of the CRF being non-Gaussian, we give a new and efficient two-step optimization procedure, with some guarantees for optimality. We utilize our global hypotheses generation procedure to produce results that exceed state-of-the-art for the challenging ``Occluded Object Dataset''.
\end{abstract}

\input{./sections/introduction.tex}

\input{./sections/relatedwork.tex}

\input{./sections/method.tex}

\input{./sections/methodGM.tex}

\input{./sections/experiments.tex}

\input{./sections/conclusion.tex}

\subsubsection*{Acknowledgements.} This work was supported by: European Research Council (ERC) under the European Union's Horizon 2020 research and innovation programme (grant agreement No 647769); German Federal Ministry of Education and Research (BMBF, 01IS14014A-D); EPSRC EP/I001107/2; ERC grant ERC- 2012-AdG 321162-HELIOS. The computations were performed on an HPC Cluster at the Center for Information Services and High Performance Computing (ZIH) at TU Dresden.


{\small
\bibliographystyle{ieee}
\bibliography{mybib}
}

\end{document}

%% file: sections/introduction.tex
\section{Introduction}
The task of estimating the 6D pose of texture-less objects has gained a lot of attention in recent years. From an application perspective this is probably due to the growing interest in industrial robotics, and in various forms of augmented reality scenarios. From an academic perspective the dataset of Hinterstoisser \etal \cite{Hinterstoisser:2012} marked a milestone, since researchers started to benchmark their efforts and progress in research started to be more measurable. In this work we focus on the following task. Given an RGB-D image of a 3D scene, in which a known 3D object is present, \ie its 3D shape and appearance is known, we would like to identify the 6D pose (3D translation and 3D rotation) of that object.
\begin{figure}[t]
\begin{center}
\includegraphics[scale=0.3]{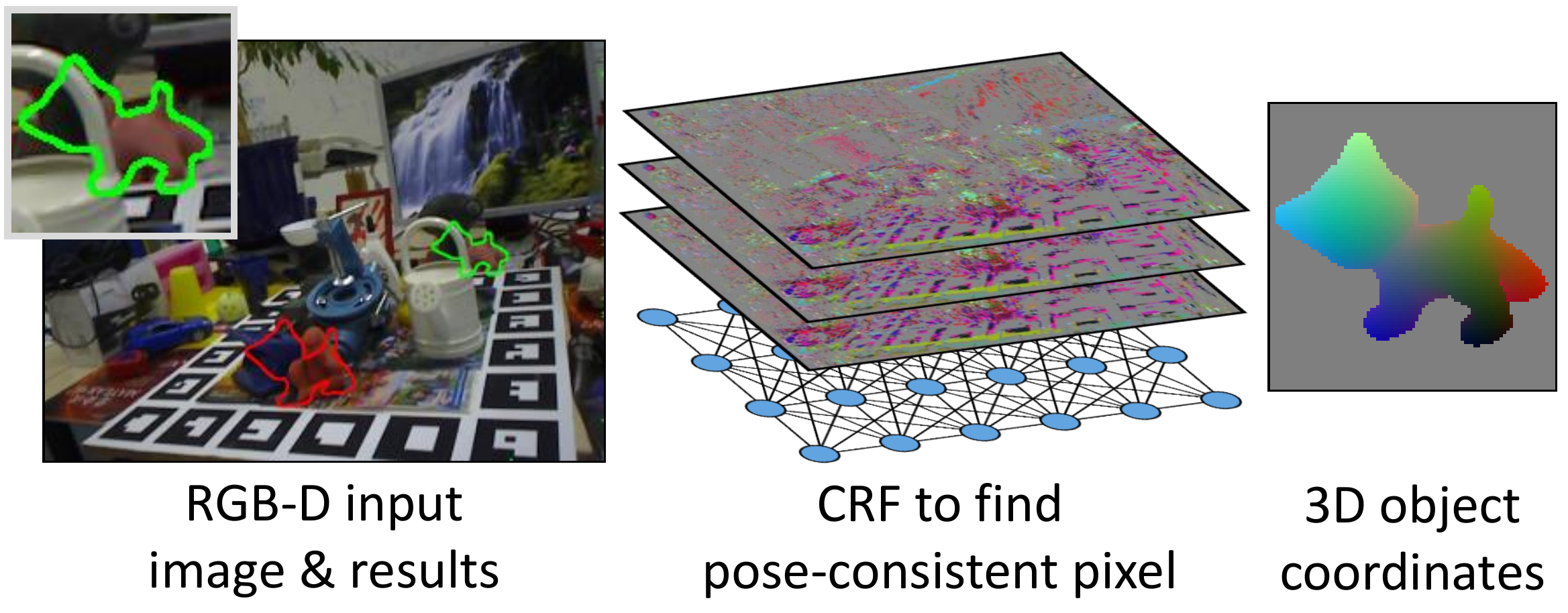}
\end{center}
   \caption{\textbf{Motivation.} Given an RGB-D input image (left) we aim at finding the 6D pose of a given object, despite it being strongly occluded (see zoom). Here our result (green) is correct, while Krull \etal \cite{Krull:2015} outputs an incorrect pose (red). The key concept of this work is to have a \textit{global}, and hence powerful, geometric check, in the beginning of the pose estimation pipeline. This is in stark contrast to \textit{local} geometric checks performed by all other methods. In a first step, a random forest predicts for each pixel a set of three possible object coordinates, \ie dense continuous part labeling of the object (middle). Given this, a fully-connected pairwise Conditional Random Field (CRF) infers globally those pixels which are consistent with the 6D object pose. We refer to those pixels as \textit{pose-consistent}. The final pose is derived from these pose-consistent pixels via an ICP-variant.}
\label{fig:teaser}
\end{figure}

\begin{table*}[]
\centering
\label{tab:categorization}
\scalebox{0.9}{
\begin{tabular}{|c|c|c|c|c|c|l|}
\hline
\textbf{Method}                                                & \begin{tabular}[c]{@{}c@{}}Intermediate\\ Representation\end{tabular}    & \begin{tabular}[c]{@{}c@{}}Hypotheses\\ Generation\end{tabular}                  & \begin{tabular}[c]{@{}c@{}}Average Number\\ of Hypotheses\end{tabular} & \begin{tabular}[c]{@{}c@{}}Hypotheses\\ Selection\end{tabular}     & \begin{tabular}[c]{@{}c@{}}Hypotheses\\ Refinement\end{tabular} & \begin{tabular}[c]{@{}l@{}}Run\\ Time\end{tabular} \\ \hline
\begin{tabular}[c]{@{}c@{}}Drost \etal \cite{Drost:2010}\\ Hinterstoisser \etal \cite{Hinterstoisser:2016}\end{tabular} & \begin{tabular}[c]{@{}c@{}}Dense Point\\ Pair Features\end{tabular}   & \begin{tabular}[c]{@{}c@{}}All local pairs\\ (large neighbourhood)\end{tabular}  & \textcolor{red}{$\sim$ $20.000$}                                     & \begin{tabular}[c]{@{}c@{}}Sub-optimal\\ search\end{tabular}       & ICP                                                             & 0.4s                                               \\ \hline
Zach \etal \cite{Zach:2015}                                                         & \begin{tabular}[c]{@{}c@{}}multiple object\\ coordinates\end{tabular} & \begin{tabular}[c]{@{}c@{}}All local triplets\\ with geometric check\end{tabular}  & \textcolor{red}{$2.000$}                                           & \begin{tabular}[c]{@{}c@{}}Optimal \wrt\\ PDA\end{tabular}         & PDA                                                             & 0.5s                                               \\ \hline
Brachmann \etal \cite{Brachmann:2014}                                                      & \begin{tabular}[c]{@{}c@{}}multiple object\\ coordinates\end{tabular} & \begin{tabular}[c]{@{}c@{}}Sampling triplets\\ with geometric check\end{tabular}   & \textcolor{red}{210}                                             & \begin{tabular}[c]{@{}c@{}}Optimal \wrt\\ Energy\end{tabular}      & \begin{tabular}[c]{@{}c@{}}ICP\\ variant\end{tabular}           & 2s                                                 \\ \hline
Krull \etal \cite{Krull:2015}                                                          & \begin{tabular}[c]{@{}c@{}}multiple object\\ coordinates\end{tabular} & \begin{tabular}[c]{@{}c@{}}Sampling triplets\\ with geometric check\end{tabular}   & \textcolor{red}{210}                                             & \begin{tabular}[c]{@{}c@{}}Optimal \wrt\\ CNN\end{tabular}         & \begin{tabular}[c]{@{}c@{}}ICP\\ variant\end{tabular}           & 10s                                                \\ \hline
Our                                                            & \begin{tabular}[c]{@{}c@{}}multiple object\\ coordinates\end{tabular} & \begin{tabular}[c]{@{}c@{}}Fully-connected CRF\\ with geometric check\end{tabular} & \textcolor{red}{0-10}                                            & \begin{tabular}[c]{@{}c@{}}Optimal \wrt\\ ICP variant\end{tabular} & \begin{tabular}[c]{@{}c@{}}ICP\\ variant\end{tabular}           & 1-3s                                               \\ \hline
\end{tabular}
}
\caption{A broad categorization of six different 6D object pose estimation methods with respect to four different computational steps: (a) Intermediate representation, (b) Hypotheses generation, (c) Hypotheses selection, (d) Hypotheses refinement, (e) Runtime. The key difference between the methods is marked in red: the number of generated hypotheses. We clearly generate least amount of hypotheses. For this we run an CRF-based hypotheses generation method which is more time-consuming and complex than in other approaches. Please note that our overall runtime is competitive. On the other hand, since we have fewer hypotheses, we can afford a more expensive ICP-like procedure to optimally select the best hypothesis. We show that we achieve results which are superior to all other methods on the challenging ``Occluded Object Dataset''. (Note PDA stands for ``projective data association''.)}
\end{table*}

Let us consider an exhaustive-search approach to this problem. We generate all possible 6D pose hypotheses, and for each hypothesis we run a robust ICP algorithm~\cite{Besl:1992} to estimate a robust geometric fit of the 3D model to the underlying data. The final ICP score can then be used as the objective function to select the final pose. This approach has two great advantages: (i) It considers all hypotheses; (ii) It uses a geometric error to prune all incorrect hypotheses.  Obviously, this approach is infeasible from a computational perspective, hence most approaches generate first a pool of hypotheses and use a geometrically motivated scoring function to select the right pose, which can be refined with robust ICP if necessary. Table \ref{tab:categorization} lists five recent works with different strategies for ``hypotheses generation'' and ``geometric selection''. The first work by Drost \etal \cite{Drost:2010}, and recently extended by Hinterstoisser \etal \cite{Hinterstoisser:2016}, has no geometric selection process, and generates a very large number of hypotheses. The pool of hypotheses is put into a Hough-space and the peak of the distribution is found as the final pose. Despite its simplicity, the method achieves very good results, especially on the challenging ``Occluded Object dataset''\footnote{http://cvlab-dresden.de/iccv2015-occlusion-challenge/}, \ie where objects are subject to strong occlusions. We conjecture that the main reason for its success is that it generates hypotheses from all local neighborhoods in the image. Especially for objects that are subject to strong occlusions, it is important to predict poses from as local information as possible. The other three approaches \cite{Brachmann:2014,Krull:2015,Zach:2015} use triplets, and are all similar in spirit. In a first step they compute for every pixel one, or more, so-called object coordinates, a 3D continuous part-label on the given object (see Fig.\ref{fig:teaser} right). Then they collect locally triplets of points, in \cite{Zach:2015} these are all local triplets and in \cite{Brachmann:2014,Krull:2015} they are randomly sampled with RANSAC. For each triplet of object coordinates they first perform a geometry consistency check (see \cite{Brachmann:2014,Krull:2015,Zach:2015} for details\footnote{For instance, the geometric check of \cite{Brachmann:2014,Krull:2015} determines whether there exists a rigid body transformation of the triplets of 3D points, given by the depth image, for the triplet of 3D points from the object coordinates.}), and if successful, they compute the 6D object pose, using the Kabsch algorithm. Due to the geometric check it is notable that the amount of generated hypotheses is substantially less for these three approaches \cite{Brachmann:2014,Krull:2015,Zach:2015} than for the previously discussed \cite{Drost:2010,Hinterstoisser:2016}. Due to this reason, the methods \cite{Brachmann:2014,Krull:2015,Zach:2015} can run more elaborate hypotheses selection procedures to find the optimal hypothesis. In \cite{Zach:2015} this is done via a so-called robust ``projective data association'' procedure, in \cite{Brachmann:2014} via a hand-crafted, robust energy, and in \cite{Krull:2015} via a CNN that scores every hypothesis. Our work is along the same direction as \cite{Brachmann:2014,Krull:2015,Zach:2015}, but goes one step forward. We presents a novel, and more powerful, geometric check, which results in even fewer hypotheses (between 0-10). For this reason we can also afford to run a complex ICP-like scoring function for selecting the best hypothesis. Since we achieve results that are better than state-of-the-art on the challenging occlusion dataset, our pool of hypotheses has at least the same quality as the larger hypotheses pool of all other methods.  Our geometric check works roughly as follows. For each pair of object coordinates a geometry-consistency measure is computed. We combine a large number of pairs into a fully-connected Conditional Random Field (CRF) model. Hence, in contrast to existing work we perform a \textit{global} geometry check and not a \textit{local} one. It is important to note that despite having a complex CRF, we are able to have a runtime which is competitive with other methods, even considerably faster than \cite{Krull:2015}. As a side note, we also achieve these state-of-the-art results with little amount of learning, in contrast to e.g. \cite{Krull:2015}.
\textbf{Our contributions} are in short:
\begin{itemize}
\item We are the first to propose a novel, \textit{global} geometry check for the task of 6D object pose estimation. For this we utilize a fully-connected Conditional Random Field (CRF) model, which we solve efficiently, although its pairwise costs are non-Gaussian and hence efficient approximation techniques like~\cite{koltun2011efficient} cannot be utilized. 
\item We give a new theoretical result which is used to compute our solutions. We show that for binary energy minimization problems, a (partial) optimal solution on a subgraph of the graphical model can be used to find a (partial) optimal solution on the whole graphical model. Proper construction of such subgraphs allows to drastically reduce the computational complexity of our method.
\item Our approach achieves state-of-the-art results on the challenging occlusion dataset, in reasonable run-time (1-3s).
\end{itemize}

%% file: sections/relatedwork.tex
\section{Related Work}
The topic of object detection and pose estimation has been widely researched in the past decade. In the brief review below, we focus only on recent works and split them into three categories. We will omit the methods \cite{Brachmann:2014,Krull:2015, Drost:2010, Zach:2015, Hinterstoisser:2016} since they were already discussed in the previous section.

\noindent {\bf Sampling-Based Methods.~}
Sparse feature based methods (\cite{Gordon:2006,Martinez:2010}) have shown good results for accurate pose estimation. They extract points of interest and match them based on a RANSAC sampling scheme. With the shift of the application scenario into robotics their popularity declined since they rely on texture.  Shotton \etal \cite{Shotton:2013:SCORF} addressed the task of camera re-localization by introducing the concept of scene coordinates. They learn a mapping from camera coordinates to world coordinates and generate camera pose hypotheses by random sampling. Most recently Phillips \etal \cite{Phillips:2016} presented a method for pose estimation and shape recovery of transparent objects where a random forest is trained to detect transparent object contours. Those edge responses are clustered and random sampling is employed to find the axis of revolution of the object. Instead of randomly selecting individual pixels we will use the entirety of the image to find pose hypotheses.

\noindent {\bf Non-Sampling-Based Methods.~}
An alternative to random sampling of pose hypotheses are Hough-voting based methods where all pixels cast a vote into a quantized prediction space (\eg 2D object center and scale). The cell with the majority of votes is taken as the winner. \cite{Gall:2011, Sun:2010} used a Hough-voting-scheme for 2D object detection and coarse pose estimation. Tejani \etal \cite{Tejani:2014} proposed an iterative latent-class Hough-voting-scheme for object classification and 3D pose estimation with RGB-D data as input. Template based methods \cite{Hinterstoisser:2012, Hinterstoisser:2012:PAMI, Huttenlocher:1993} have also been applied to the task of pose estimation. To find the best match the template is scanned across the image and a distance metric is computed at each position. Those methods are harmed by clutter and occlusion which disqualifies them to be applied to our scenario. In our approach each pixel is processed, but instead of them voting individually we find pose-consistent pixel-sets by global reasoning.

\noindent {\bf Pose Estimation using Graphical Models.~}
In an older piece of work the pose of object categories was found in images either in 2D~\cite{Winn:2006} or in 3D~\cite{Hoiem:2007}. They also use the key concept of discretized object coordinates for object detection and pose estimation. The MRF-inference stage for finding pose-consistent pixels is closely related to ours. Foreground pixels are accepted when the layout consistency constraint (where layout consistency means that neighboring pixels should belong to the same part) is satisfied. However since the shape of the object is unknown, the pairwise terms are not as strong as in our case. The closest related work to our is Bergholdt \etal \cite{Bergtholdt:2009}. They use the same strategy of discriminatively modeling the local appearance of object parts and globally inferring the geometric connections between them. To detect and find the pose of articulated objects (faces, human spines, human poses) they extract feature points locally and combine them in a probabilistic, fully-connected, graphical model. However they rely on a exact solution to the problem while a partial optimal solution is sufficient in our case. We therefore employ a different approach to solve the task.

%% file: sections/method.tex
\section{Method - Overview}
Before we describe our work in detail, we will introduce the task of 6D pose estimation formally and provide a high-level overview of our method. The objective is to find the 6D pose $\mathbf{H}_c = [R_c|t_c]$ of object $c$, with $R_c$ ($3 \times 3$ matrix) describing a rotation around the object center and $t_c$ ($3 \times 1$ vector) representing the position of the object in camera space. The pose $\mathbf{H}_c$ transforms each point in object coordinate space $\mathbf{y} \in \mathcal{Y} \subseteq \mathbb{R}^3$ into a point in camera space $\mathbf{x} \in \mathcal{X} \subseteq \mathbb{R}^3$.
\begin{figure*}
\begin{center}
\includegraphics[scale=0.45]{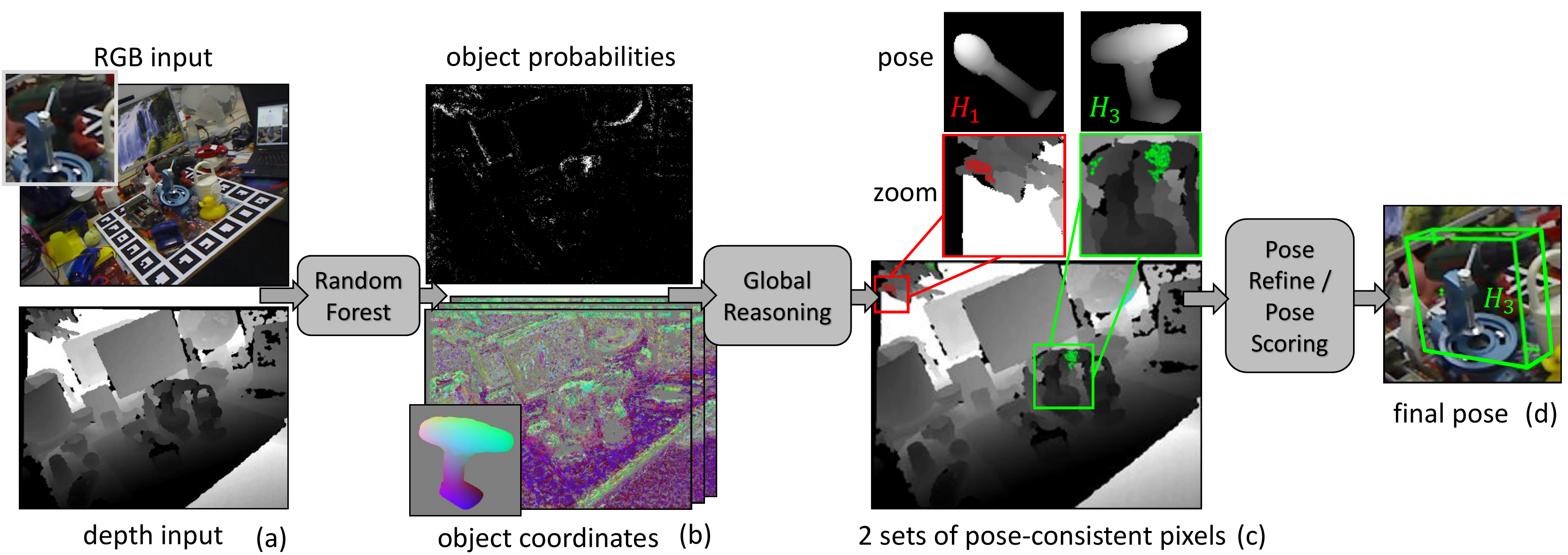}
\end{center}
   \caption{\textbf{Our pipeline:} Given an RGB-D image (a) a random forest provides two predictions: object probabilities and object coordinates (b). In a second stage our novel, fully-connected CRF infers pose-consistent pixel-sets (see zoom) (c). In the last stage, pose hypotheses given by pose-consisent pixels of the CRF are refined and scored by an ICP-variant. The pose with the lowest score is given as output (d).}
\label{fig:method-overview}
\end{figure*}

Our algorithm consists of three stages (see Fig. \ref{fig:method-overview}). In the first stage (Sec. \ref{sec:met-rf}) we densely predict object probabilities and object coordinates using a random forest. Instead of randomly sampling pose hypotheses as \eg in \cite{Brachmann:2014} we use a graphical model to globally reason about hypotheses inliers. This second stage is described in Section \ref{sec:met-inf} roughly and in Section~\ref{sec:method-gm} in detail. In the final stage (Sec. \ref{sec:met-icp}) we refine and rank our pose hypotheses to determine the best estimate.
\input{./sections/meth_RF.tex}

\input{./sections/meth_GR.tex}
\input{./sections/meth_ICP.tex}

%% file: sections/meth_RF.tex
\subsection{Random Forest}
\label{sec:met-rf}
We use the random forests from Brachmann \etal \cite{Brachmann:2014}\footnote{We kindly thank the authors for providing them}. Each tree $T$ of the forest $\mathcal{T}$ predicts for each pixel an object probability and an object coordinate. As mentioned above, an object coordinate corresponds to a 3D point on the surface of the object. In our case we have $T=3$. As in \cite{Brachmann:2014} the object probabilities from multiple trees that are combined to one value using Bayes rule. This means that for a pixel $i$ and object $c$ we have the object probability $p_c(i)$. The object probabilities can be seen as a soft segmentation mask.

%% file: sections/meth_GR.tex
\subsection{Global Reasoning}
\label{sec:met-inf}
In general, to estimate the pose of a rigid object, a minimal set of three correspondences between 3D points on the object and in the 3D scene is required~\cite{Kabsch:1976}. The 3D points on the object, \ie in the object coordinate system, are predicted by the random forest. One possible strategy is to generate such triplets randomly by RANSAC~\cite{Fischler:1981}, as proposed in \cite{Brachmann:2014}. However, this approach has a serious drawback: the number of triples which must be generated by RANSAC in order to have at least a correct triple with the probability of 95\%, is very high. Assuming that $n$ out of $N$ pixels contain correct correspondences, the total number of samples is $\frac{\log(1-0.95)}{\log(1-(1-n/N)^3)}$. For $n/N=0.005$, which corresponds to a state-of-the-art local classifier, this constitutes $\sim24.000.000$ RANSAC iterations. Therefore, we address this problem with a different approach. Our goal is to assign to each pixel either one of the possible correspondence candidates, or an ``outlier'' label. We achieve this by formalizing a graphical model where each pixel is connected to every other pixel with a pairwise term. The pairwise term encodes a geometric check which is defined later. The optimization problem of this graphical model is discussed in Sec. \ref{sec:gm-pose}.

%% file: sections/meth_ICP.tex
\subsection{Refinement and Hypothesis Scoring}
\label{sec:met-icp}
The output of the optimization of the graphical model is a collection of pose-consistent pixels where each of those pixels has a unique object coordinate. The collection is clustered into sets. In the example in Fig. \ref{fig:method-overview}(c) there are two sets (red, green). Each set provides one pose hypothesis. These pose hypotheses are refined and scored using our ICP-variant. In order to be robust to occlusion we only take the pose-consistent pixels within the ICP \cite{Besl:1992} for fitting the 3D model.

%% file: sections/methodGM.tex
\section{Method - Graphical Model}
\label{sec:method-gm}
After a brief introduction to graphical models (Sec. \ref{sec:energyForm}), we define our graphical model used for object pose estimation (Sec. \ref{sec:gm-pose}). This is a fully-connected graph where each node has multiple labels, here 13. The globally optimal solution of this problem gives a pose-consistent (inlier) label to only those pixels that are part of the object, ideally. Since our potential functions are non-Gaussian the optimization problem is very challenging. We solve it, very efficiently, in a two stage procedure, with some additional guarantees. The first stage conservatively prunes those pixels that are likely not inliers. This is done with a sparsely connected graph and TRW-S \cite{Kolmogorov:Pami2006} as inference procedure (Sec. \ref{sec:stage-one}).  The second stage (Sec. \ref{sec:stage-two} - \ref{sec:met-candiates}) describes an efficient procedure for solving the problem with only the inlier candidates remaining.  We proove that by splitting this problem further into subproblems, in a proper way, a solution to one of these subproblems is guaranteed the optimal solution of the original problem.  

\subsection{Energy Minimization}
\label{sec:energyForm}
Let $G=(V,E)$ be an undirected graph with a finite {\em set of nodes} $V$ and {\em a set of edges} $E\in {V\choose 2}$. 
With each node $u\in V$  we associate a finite {\em set of labels} $L_u$. Let $\prod$ stand for the Cartesian product. The set $\BL=\prod_{u\in V}L_u$ is called {\em the set of labelings}. Its elements $l \in \BL$, called {\em labelings}, are vectors $l=(l_u\in L_u\colon {u\in V})$ with $|V|$ coordinates, where each one specifies a label assigned to the corresponding graph node. 
For each node {\em a unary cost function} $\theta_u\colon L_u\to\BR$ is defined. Its value $\theta_u(l_u)$, $l_u\in L_u$ specifies the cost to be paid for assigning label $l_u$ to node $u$. For each two neighboring nodes $\{u,v\}\in E$ {\em a pairwise cost function} $\theta_{uv}\colon L_u\times L_v \to \BR$ is defined. Its value $\theta_{uv}(l_u,l_v)$ specifies compatibility of labels $l_u$ and $l_v$ in the nodes $u$ and $v$, respectively. 
The triple $(G,\BL,\theta)$ defines {\em a graphical model}.

The {\em energy} $E_V(l)$ of a labeling $l\in \BL$ is a total sum of the corresponding unary and pairwise costs 
\begin{equation}
 E_V(l):= \sum_{u\in V}\theta_u(l_u) + \beta \sum_{uv\in E}\theta_{uv}(l_u,l_v)\,.
\end{equation}
Finding a labeling with the lowest energy value constitutes {\em an energy minimization problem}. Although this problem is NP-hard, in general, a number of efficient approximative solvers exist, see~\cite{kappes-2015-ijcv} for a recent review.

\subsection{Pose Estimation as Energy Minimization}
\label{sec:gm-pose}
Consider the following energy minimization problem:
\begin{itemize}
 \item The set of nodes is the set of pixels of the input image, i.e., each graph node corresponds to a pixel. To be precise, we scale down our image by a factor of two for faster processing, \ie each graph node corresponds to $2\times 2$ pixels.
 \item Number of labels in every node is the same. The label set $L_u:=\hat L_u\cup\{\textrm{o}\}$ consists of two parts, a subset $\hat L_u$ of correspondence proposals and a special label $\textrm{o}$. In total, each node is assigned $13$ labels: The forest $\mathcal{T}$ provides $3$ candidates for object coordinates in each pixel, $2\times 2$ pixels result in $12$ labels, and the last label is the ``outlier''.

Each label from the subset $\hat L_u$  corresponds a 3D coordinate on the object. Therefore, we will associate such labels $l_u$ with 3D vectors and assume vector operations to be well-defined for them. Unary costs $\theta_u(l_u)$ for these labels are 
 set to $(1-p_c(u))\alpha$, where $p_c(u)$ is defined in Section~\ref{sec:met-rf} and $\alpha$ is a hyper-parameter of our method. We will call the labels from $\hat L_u$ {\em inlier labels} or simply {\em inlier}. 

The special label $\textrm{o}$ denotes a situation in which the corresponding node does not belong to the object, or none of the labels in $\hat L_u$ predicts a correct object coordinate. We call $\textrm{o}$ the ``outlier label''. Unary costs for the outlier labels are: $\theta_u(\textrm{o})=\frac{\sum p_c(u)\alpha}{12}$, $u\in V$.

Let us define pose-consistent pixels. If a node, comprising of $2 \times 2$ pixels, is an inlier then the pixel with the respective label is defined as pose-consistent. The remaining three pixels are not pose-consistent and are ignored in the hypotheses selection stage. Also all pixels for which the node has an outlier label are not pose-consistent.

 \item Let $x_u$ and $x_v$ be $3D$ points in the camera coordinate system, corresponding to the nodes $u$ and $v$ in the scene. For any two inlier labels $l_u\in\hat L_u$ and $l_v\in \hat L_v$ we assign the pairwise costs as follows
{\small
\begin{equation}\label{eq:pairwise-costs}
\hspace{-15pt}\theta_{uv}(l_u,l_v)=
\begin{cases}
 \Big|\Vert l_u-l_v\Vert - \Vert x_u-x_v\Vert\Big|, & \Vert x_u-x_v\Vert \le D\\
 \infty, & \text{otherwise}.
\end{cases}
\end{equation}
}
That is, $\theta_{uv}(l_u,l_v)$ is equal to the absolute difference of distances between points $l_u, l_v$ on the object and $x_u, x_v$ in the scene (see Fig.~\ref{fig:method-binaries}) if the latter difference does not exceed the object size $D$.

Additionally, we define $\theta_{uv}(l_u,\textrm{o})=\theta_{uv}(\textrm{o},l_v)=\gamma$ for $l_u\in L_u$, $l_v\in L_v$. Here $\gamma$ is another hyper-parameter of our method. A sensible setting is $\gamma=0$, however, we will choose $\gamma>0$ in parts of the optimization (see details below). We also assign ${\theta_{uv}(\textrm{o},\textrm{o})=0}$, for all $\{u,v\}\in E$.
\item The graph $G$ is fully-connected, i.e., any two nodes $u, v\in V$ are connected by an edge $\{u,v\}\in E$.
\end{itemize}

\begin{figure}
\begin{center}
\includegraphics[scale=0.3]{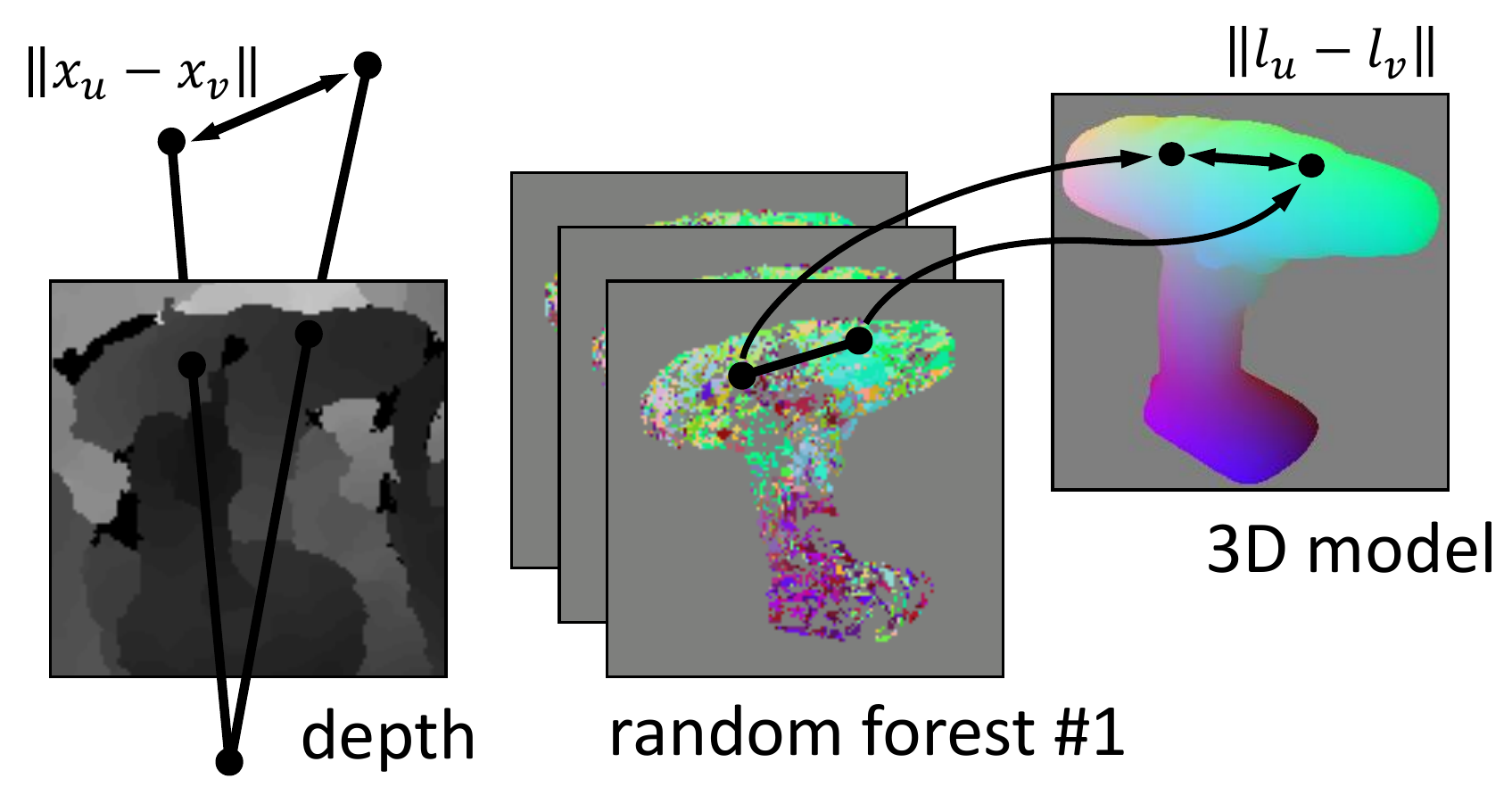}
\end{center}
   \caption{Visualization of our binary potential as defined in Eq. \ref{eq:pairwise-costs}.}
\label{fig:method-binaries}
\end{figure}

Given a labeling $l\in \BL$ we will speak about {\em inlier} and {\em outlier} nodes as those labeled with inlier or outlier labels, respectively.

The energy of any labeling is a sum of (i) the total unary costs for inlier labels, (ii) total geometrical penalty of the inlier labels, and (iii) total cost for the outlier labels.
A labeling with the minimal energy corresponds to a geometrically consistent subset of coordinate correspondences with a certain confidence for the local classifiers.
We believe, there are such hyper-parameter settings that these coordinates would provide approximately correct object poses.

\textbf{Why a fully-connected graph?} At the first glance, one could reasonably simplify the energy minimization problem described above by considering a sparse, \eg grid-structured graph. In this case the pairwise costs would control not all pairs of inlier labels, but only a subset of them, which may seem to be enough for a selection of inliers defining a good quality correspondence. Unfortunately, such a simplification has a serious drawback, nicely described in~\cite{Bergtholdt:2009}: As soon as the graph is not fully connected, it tends to select an optimal labeling, which contains separated ``islands'' of inlier nodes, connecting to other ``inlier-islands'' only via outlier nodes. Such a labeling may contain geometrically independent subsets of inlier labels, which may ``hallucinate'' the object in different places of the image. Moreover, from our experience many of such ``islands'' contain less than three nodes, which increases the probability for pairwise geometrical costs to be low just by chance.

Concerning energy minimization. Apart from the very special case with Gaussian potentials (like e.g.~\cite{koltun2011efficient}) even solving approximately an energy minimization problem on the fully-connected graph with $320\times 240$ nodes (which corresponds to the size of our discretized input image) is in general an infeasible task for modern methods. Therefore, we suggest here {\em a problem-specific}, but {\em very efficient} two-stage procedure for generating approximative solutions of the considered problem. In a first stage (Sec.~\ref{sec:stage-one}) we reduce the size of the optimization problem, in the second (Sec.~\ref{sec:stage-two}) we generate solution candidates.

\subsection{Stage One: Problem Size Reduction}
\label{sec:stage-one}
Despite what is discussed above about having a fully-connected graph, we used such a sparse graphical model to reduce the number of possible correspondence candidates. 
An optimal labeling of this sparse model provides us with a set of inlier nodes, which hopefully contain the true inliers. On the second stage of our optimization procedure, described below, we build several fully-connected graphs from these nodes. For the sparse graph we use the following neighborhood structure: we connect each node to 48 closest nodes excluding the closest 8. We believe that the distance measure between the closest nodes is very noisy.


We assign a positive value to the parameter $\gamma$ penalizing transitions between inlier and outlier labels. This decreases the number of ``inlier islands'' by increasing the cost of the transition. 
We approximately solved this sparse problem with the TRW-S algorithm~\cite{Kolmogorov:Pami2006}, which we run for $10$ iterations. We found the recent implementation~\cite{Shekhovtsov:2015} of this algorithm to be up to $8$ times faster than the original one~\cite{Kolmogorov:Pami2006} for our setting.

\begin{figure}
\begin{center}
\includegraphics[scale=0.45]{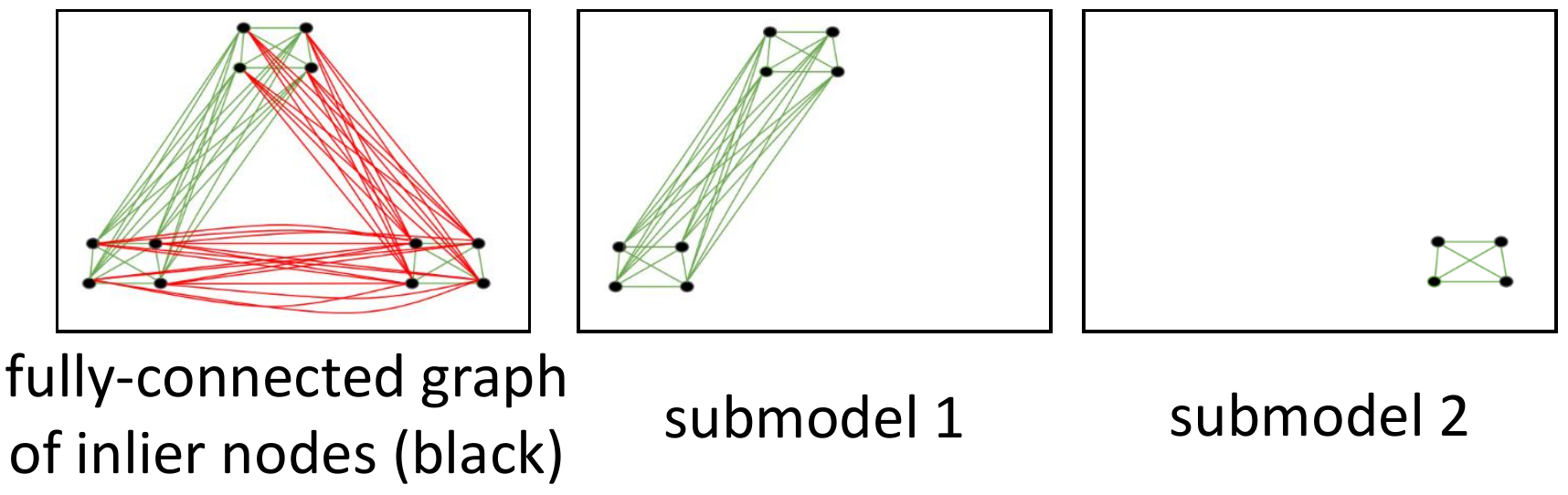}
\end{center}
\caption{ \textbf{Illustrating Optimization Stage Two.} (Left) the black pixels are all those pixels which were labeled as inliers, (potentially pose-consistent) in the first stage of the optimization. The first stage is opportunistic in the sense that wrong inliers may still be present. The goal of the second stage is to determine exactly the true inliers, from which we will determine the final pose. For this we have to solve the fully-connected graph shown, where each pixel has two labels, being an inlier $(1)$ or outlier $(0)$. Here the red links mark pairwise terms which contain $\infty$ values. Unfortunately, state of the art solvers struggle with this problem, due to the presence of red links. We solve this by solving two (in practice many more) submodels (middle, right) that contain no red links.  Each sub-problem produces a partial optimal solution $\{0,1,?\}$, where nodes that do not belong to the submodel are labeled $0$. We can now guarantee that one of the partial optimal solution is the partial optimal solution of the full graph shown on the left.}
\label{fig:method-2stage}
\end{figure}

\subsection{Stage Two: Generation of Solution Candidates}\label{sec:stage-two}
\noindent {\bf Fully-Connected Graphical Model.~}
As mentioned above, in the second stage we consider {\em a fully-connected} graphical model, where the node set contains only inlier nodes from the solution of the sparse problem. Moreover, to further reduce the problem size, we reduce the label set in each node to only two labels $L_u:=\{0,1\}$, where the label $0$ corresponds to an outlier and the label $1$ corresponds to the label associated with the node in the solution of the sparse problem. The unary and pairwise costs are assigned as before, but the hyper-parameters $\alpha$, $\beta$ and $\gamma$ are different. In particular $\gamma = 0$  since there is no reason to penalize transitions between inlier and outlier on this stage. Further, we will refer to $(G,\BL,\theta)$ defined above, as to master ({\bf fully-connected}) model $F$. 

Although such problems usually have a much smaller size (the solution of the sparse problem typically contains $20$ to $500$ inliers) our requirements to a potential solver are much higher at this stage. Whereas in the first stage we require only that the set of inlier nodes contains enough of correct correspondences, the inliers obtained on the second stage must be {\em all} correct (have small geometrical error). Incorrect correspondences may deteriorate the final pose estimation accuracy. Therefore the quality of the solution becomes critical on this stage. 
Although problems of this size are often feasible for exact solvers, obtaining an exact solution may take multiple minutes or even hours. Therefore, we stick to the methods delivering only {\em a part of an optimal solution (partial optimal labeling)}, but being able to do this in a fraction of seconds, or seconds, depending on the problem size. Indeed, it is sufficient to have only three inlier to estimate the object pose.

\noindent {\bf Partial Labeling.~}
Under a partial labeling we understand a vector $l\in \{0,1,?\}^{|V|}$ with only a subset $V'\subset V$ of coordinates assigned a value $0$ or $1$. The rest of coordinates take a special value $?$ = ``unlabeled''. Partial labeling is called {\em partial optimal labeling}, if there exists {\em an optimal} labeling $l^*\in \BL$ such that $l^*_u=l_u$ for all $u\in V'$.

There are a number of efficient approaches addressing partial optimality (obtaining partial optimal labelings) for discrete graphical models for both multiple~\cite{Swoboda:2014,Shekhovtsov:2015} and two-label cases~\cite{kolmogorov2007minimizing,wang2016relaxation}. We refer to~\cite{shekhovtsov-14} for an extensive overview. For problems with two labels the standard partial optimality method is QPBO~\cite{kolmogorov2007minimizing}, which we used in our experiments. 

All partial optimality methods are based on sufficient optimality conditions, which have to be fulfilled for a partially optimal labeling. However, as it directly follows from~\cite[Prop.1]{Swoboda2016}, these conditions can hardly be fulfilled for label $l_u$ in a node $u$, if for some neighboring node $v\colon \{u,v\}\in E$ the difference between the smallest pairwise potential ``attached'' to the label $l_u$, $\min_{l_v\in L_v}\theta_{uv}(l_u,l_v)$ and the largest one  $\max_{l_v\in L_v}\theta_{uv}(l_u,l_v)$ is very large. In our setting this is the case, \eg , if for two nodes $u$ and $v$ (connected by an edge as any pair in a fully-connected graph) it holds $\Vert x_u-x_v\Vert > D$, see~\eqref{eq:pairwise-costs}. Existence of such infinite costs leads to deterioration of the QPBO results: in many cases the returned partial labeling contains less than $3$ labeled nodes, which is not sufficient for pose estimation. 

To deal with this issue, we propose a novel method to find {\em multiple} partial labelings: We consider a set of induced submodels (see Definition~\ref{def:induced-submodel} below) and find a partial optimal solution for each of them. We guarantee, however, that {\em at least} one of these partial labelings is a partial optimal one {\em for the whole graphical model} and not only for its submodel.
Considering submodels allows to significantly reduce the number of node pairs $\{u,v\}$ with $\theta_{uv}(1,1)=\infty$. In its turn, it leads to many more nodes being marked as partially optimal by QPBO and therefore, provides a basis for a high quality pose reconstruction (see Fig. \ref{fig:method-2stage}).

The theoretical background for the method is provided in the following subsection.

\subsection{On Optimality of Subproblem Solutions for Binary Energy Minimization}
\label{sec:met-opt}

Let $G=(V,E)$ be a graph and $V'\subset V$ be a subset of its nodes. A subgraph $G'=(V',E')$ is called {\em induced} w.r.t.~$V'$, if $E'=\{\{u,v\}\in E\colon u,v\in V'\}$ contains all edges of $E$ connecting nodes within $V'$. 
\begin{definition}\label{def:induced-submodel}
 Let $M=(G,\BL,\theta)$ be a graphical model with $G=(V,E)$ and $\BL=\prod_{u\in V} L_u$. A graphical model $M'=(G',\BL',\theta')$ is called {\em induced} w.r.t.\ $V'\subseteq V$ if \\
\indent $\bullet$  $G'$ is an induced subgraph of $G$ w.r.t. $V'$. \\
\indent $\bullet$ $\BL'=\prod_{u\in V'}L_u$. \\
\indent $\bullet$ $\theta'_u=\theta_u$ for $u\in V'$ and $\theta'_{uv}=\theta_{uv}$ for $\{u,v\}\in E'$. 
\end{definition}

\begin{proposition}\label{prop:proof-of-submodel}
 Let $M=(G,\BL,\theta)$ be a graphical model, with $G=(V,E)$, $\BL=\{0,1\}^{|V|}$ and $\theta$ such that 
 \begin{equation}\label{eq:zero-pairwise-potentials}
\hspace{-8pt}  \theta_{uv}(0,1)=\theta_{uv}(1,0)=\theta_{uv}(0,0)=0\quad \forall \{u,v\}\in E\,.
 \end{equation}
 Let $\hat l\in \BL$ be an energy minimizer of $M$ and ${\hat V:=\{u\in V\colon \hat l_u=1\}}$.\\
 Let $M'=(G',\BL',\theta')$ be an induced model w.r.t. some $V'\supseteq \hat V$ and $l'$ be an energy minimizer of $M'$.
 Then there exists a minimizer $l^*$ of energy of $M$, such that $l'_u=l^*_u$ for all $u\in V'$.
\end{proposition}

\begin{proof}
$E_{V}(\hat l) = E_{V'}(\hat x_{V'}) + E_{V\backslash {V'}}(\hat x_{V\backslash {V'}}) \ge E(l') + E_{V\backslash {V'}}(\overline{0})$.
Since $x_{V\backslash {V'}}=\overline 0$ due to~\eqref{eq:zero-pairwise-potentials}, the equality holds.
The inequality holds by definition of $l'$. Let us consider the labeling $l^*:=(l',\overline{0})$ constructed by concatenation of $l'$ on $V'$ and $\overline 0$ on $V\backslash V'$. Its energy is equal to the right-hand-side of the expression, due to~\eqref{eq:zero-pairwise-potentials}. Since $\hat l$ is an optimal labeling, the inequality holds as equality and the labeling $l^*$ is optimal as well. It finalizes the proof.
\end{proof}

\begin{corollary}
 Let under condition of Proposition~\ref{prop:proof-of-submodel} $l'$ be a partial optimal labeling for $M'$. Then it is partial optimal for $M$.
\end{corollary}
Note, since pairwise costs of {\em any} two-label (pairwise) graphical model can be easily transformed to the form~\eqref{eq:pairwise-costs}, see e.g.~\cite{kolmogorov2007minimizing}, Proposition~\ref{prop:proof-of-submodel} is generally applicable to all such models.

\subsection{Obtaining Candidates for Partial Optimal Labeling}
\label{sec:met-candiates}
To be able to use Proposition~\ref{prop:proof-of-submodel} we need a way to characterize possible optimal labelings for the master model $F$ (defined in Section~\ref{sec:stage-two}) to be able to generate possible sets $V'$ containing all inlier nodes of an optimal labeling. Indeed, this characterization is provided by the following proposition: 
\begin{proposition}
Let $l^*$ be an optimal solution to the fully-connected problem described above. Then for any two inlier nodes $u$ and $v$, $l^*_u=l^*_v=1$, it holds $\Vert x_u-x_v\Vert \le D$ or, in other words, $\theta_{uv}(l^*_u,l^*_v) < \infty$. 
\end{proposition}
This proposition has a trivial proof: as soon as there is a labeling with a finite energy (e.g. $l_u=0$ for all $u\in V$), an optimal labeling can not have an infinite one.

An implication of the proposition is quite clear from the applied point of view: all inlier nodes must be placed within a circle with a diameter equal to the maximal linear size of the object. Combining this observation with Proposition~\ref{prop:proof-of-submodel}, we will generate a set of submodels, which contain all possible subsets of nodes satisfying the above condition. 

A simple, yet inefficient way to generate all such submodels, is to go over all nodes $u$ of the graph $G$ and construct a subproblem $M_u$ induced by nodes, which are placed at most at the distance $D$ of $u$. A disadvantage of this method is that one gets as many as $|V|$ subproblems, which leads to the increased runtime and too many almost equal submodels. Instead, we consider all connected inlier components obtained on the first stage as a result of the problem reduction. We remove all components with the size less than $3$, because, as we found experimentally, they mostly represent only noise.  We enumerate all components, i.e., assign a serial number to each. For each component $f$ we build a fully-connected submodel, which includes itself and all components with bigger serial number within the distance $D$ from all nodes of $f$. Such an approach usually leads to at most $20$ submodels and most of them get more than $3$ partial optimal labels by QPBO.

%% file: sections/experiments.tex
\section{Experiments}
\begin{figure*}[h!t!]
\centering
\begin{minipage}[b]{0.28\textwidth}
	\includegraphics[width=\linewidth]{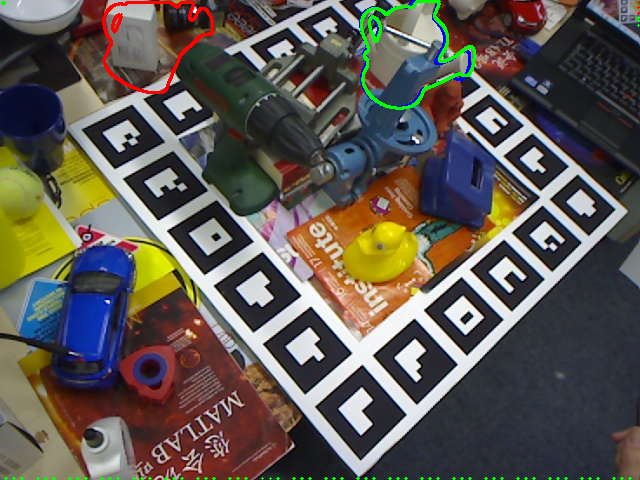}
\end{minipage}
\begin{minipage}[b]{0.28\textwidth}
	\includegraphics[width=\linewidth]{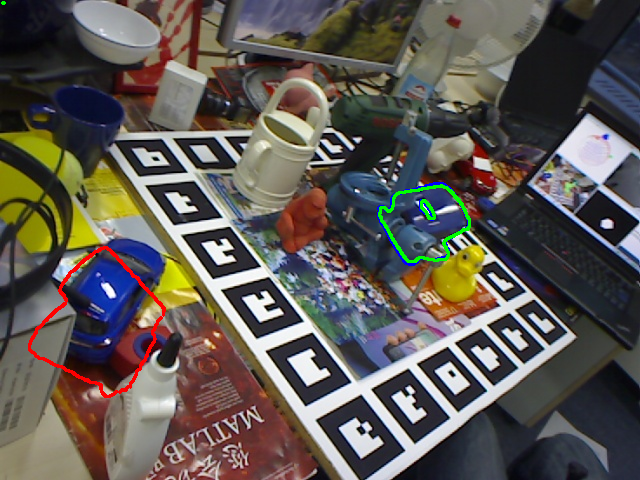}
\end{minipage}
\begin{minipage}[b]{0.28\textwidth}
	\includegraphics[width=\linewidth]{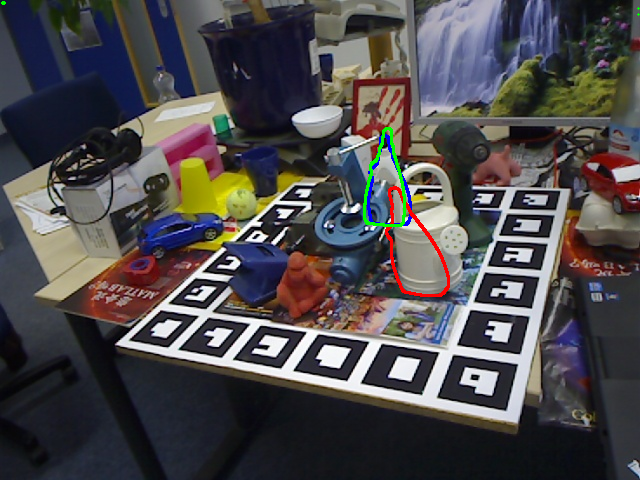}
\end{minipage}
\caption{Qualitative results of our method on the ``Occluded Object Dataset''~\cite{Brachmann:2014}. Results of our method are depicted as green silhouettes, the ground truth pose is shown as a blue silhouette and results of the method by Krull \etal \cite{Krull:2015} are shown as red silhouettes. Note, since these results shows correct poses of our method the green silhouette is on top of the blue one.}
\label{fig:results}
\end{figure*}

\begin{figure}[h]
\begin{center}
\includegraphics[scale=0.5]{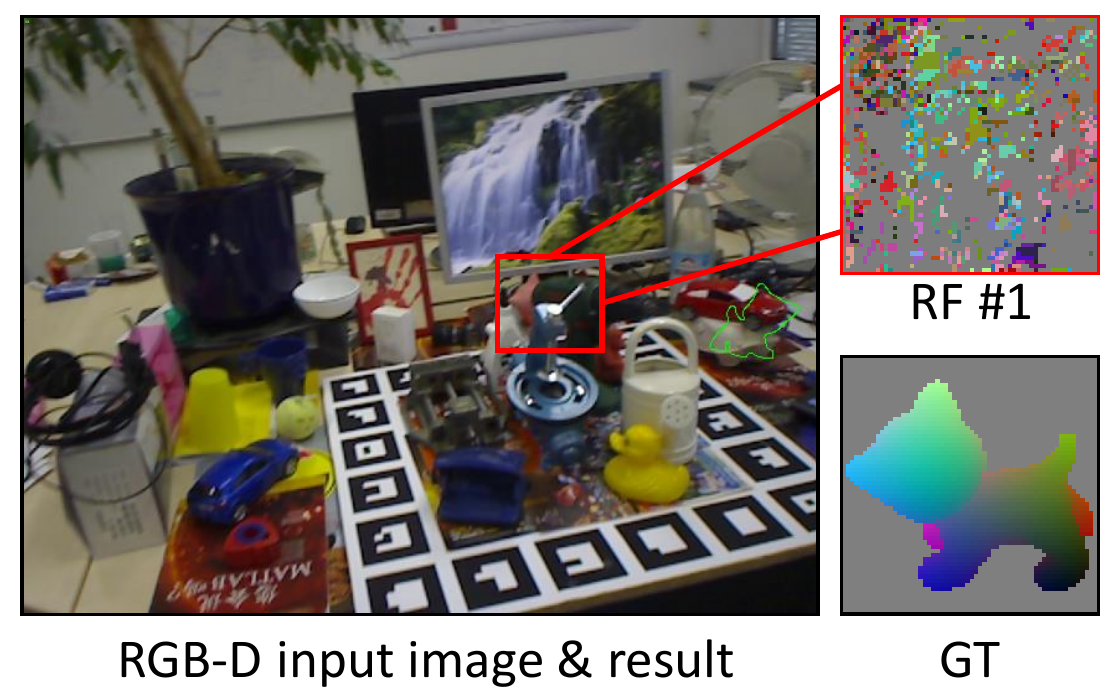}
\end{center}
   \caption{\textbf{Failure case.} We use the random forest from \cite{Brachmann:2014} which were trained on image patches of non-occluded objects. Hence they can only handle a moderate level of occlusion. In case of strong occlusion they fail to predict good object coordinates. In the illustrated example, a wrong pose is predicted (green silhouette) and the object coordinates are also wrong (see zoom). In future work, this problem can be mitigated for instance by training on image patches that contain occlusions.}
\label{fig:results-fail}
\end{figure}

We evaluated our method on a publicly available dataset. We will first introduce the dataset and then the evaluation protocol (Sec.~\ref{sec:dataset}). After that, we quantitatively compare our work with three competitors, and also present qualitative results (Sec.~\ref{sec:results}).
\begin{table*}[]
\centering
\label{table-results}
\begin{tabular}{c|cccc}
Method       & \begin{tabular}[c]{@{}c@{}}Our method\end{tabular} & \begin{tabular}[c]{@{}c@{}}Hinterstoisser et al.\cite{Hinterstoisser:2016}\end{tabular} & \begin{tabular}[c]{@{}c@{}}Krull et al.\cite{Krull:2015}\end{tabular} & \begin{tabular}[c]{@{}c@{}}Brachmann et al.\cite{Brachmann:2014}\end{tabular} \\ \hline
Object       & \multicolumn{4}{c}{Scores}                                                                                                                                                                                                                   \\ \hline
Ape          & 80.7\%                                               & \textbf{81.4\%}                                                 & 68.0\%                                                 & 53.1\%                                                     \\
Can          & 88.5\%                                               & \textbf{94.7\%}                                                 & 87.9\%                                                 & 79.9\%                                                     \\
Cat          & \textbf{57.8\%}                                      & 55.2\%                                                          & 50.6\%                                                 & 28.2\%                                                     \\
Driller      & \textbf{94.7\%}                                      & 86.0\%                                                          & 91.2\%                                                 & 82.\%                                                      \\
Duck         & 74.4\%                                               & \textbf{79.7\%}                                                 & 64.7\%                                                 & 64.3\%                                                     \\
Eggbox       & 47.6\%                                               & \textbf{65.5\%*}                                                & 41.5\%                                                 & 9.0\%                                                      \\
Glue         & \textbf{73.8\%}                                      & 52.1\%                                                          & 65.3\%                                                 & 44.5\%                                                     \\
Hole Puncher & \textbf{96.3\%}                                      & 95.5\%                                                          & 92.9\%                                                 & 91.6\%                                                     \\ \hline
Average      & \textbf{76.7\%}                                               & 76.2\%                                                          & 70.3\%                                                 & 56.6\%                                                    
\end{tabular}
\caption{Quantitative comparision of \cite{Brachmann:2014}, \cite{Krull:2015}, \cite{Hinterstoisser:2016} and our approach for all objects in the challenging ``Occluded Object Dataset''. *The number for the Eggbox differs from \cite{Hinterstoisser:2016} since they did not consider all images of the sequence (private e-mail exchange with the authors).}
\end{table*}

\subsection{Dataset}
\label{sec:dataset}
To evaluate our method, we use the publicly available dataset of Brachmann \etal \cite{Brachmann:2014}, known as ``Occluded Object Dataset''\footnote{http://cvlab-dresden.de/iccv2015-occlusion-challenge/}. 
This dataset was presented in \cite{Brachmann:2014} and is an extension of \cite{Hinterstoisser:2012}. They annotated the ground truth pose for 8 objects in 1214 images with various degrees of object occlusions.

To evaluate our method we use the criteria from \cite{Hinterstoisser:2012}. This means we measure the percentage of correctly estimated poses for each object. To determine the quality of an estimated pose 
we calculate the average distance of each point with respect to the estimated pose and the ground truth pose. The pose is accepted if the average distance is below 10\% of the object diameter.

To find good parameters for our graphical model we created a validation set, which we will make publicly available. For this we annotated an additional image sequence (1235 images) of \cite{Hinterstoisser:2012} containing 6 objects. The final set of parameters for stage one is $\alpha = 0.21, \beta = 23.1, \gamma = 0.0048$ and stage two is $\alpha = 0.2, \beta = 2.0, \gamma = 0.0$.

\subsection{Results}
\label{sec:results}
In the following we compare to the methods of Brachmann \etal \cite{Brachmann:2014}, Krull \etal \cite{Krull:2015} and to the recently published state-of-the-art method of Hinterstoisser \etal \cite{Hinterstoisser:2016}.
Results are shown in Table \ref{table-results}. We achieve an average accuracy of $76.7\%$ over all objects, which is $0.4\%$ better than the current state-of-the-art method of Hinterstoisser \etal \cite{Hinterstoisser:2016}. With respect to individual objects our method performs best on four objects and \cite{Hinterstoisser:2016} on the other four. In comparison with \cite{Brachmann:2014} and \cite{Krull:2015} we achieve an improvement of $20.1\%$ and $6.4\%$ respectively. Since these two methods use the same random forest, as we do, the benefits of using global reasoning can be seen. See Fig. \ref{fig:results} for qualitative results.

%% file: sections/conclusion.tex
\section{Conclusion and Future Work}
In this work we have focused on the pose-hypothesis generation step, which is part of many pipelines for 6D object pose estimation. For this, we introduced a novel, global geometry check in form of a fully connected CRF. Since this direct optimization on the CRF is hardly feasible, we present an efficient two-step optimization procedure, with some guarantees on optimality. There are many avenues for future work. An obvious next step is to improve on the regression procedure for object coordinates, \eg by replacing the random forests with a convolutional neural network.  

%% file: Global_Hypothesis_Generation_for_6D_Object_Pose_Estimation_-_Arxiv.bbl
\begin{thebibliography}{10}\itemsep=-1pt

\bibitem{Bergtholdt:2009}
M.~Bergtholdt, J.~Kappes, S.~Schmidt, and C.~Schn{\"o}rr.
\newblock A study of parts-based object class detection using complete graphs.
\newblock {\em International Journal of Computer Vision}, 87(1):93, 2009.

\bibitem{Besl:1992}
P.~J. Besl and N.~D. McKay.
\newblock A method for registration of 3-d shapes.
\newblock {\em IEEE Trans. Pattern Anal. Mach. Intell.}, 14(2):239--256, 1992.

\bibitem{Brachmann:2014}
E.~Brachmann, A.~Krull, F.~Michel, J.~Shotton, S.~Gumhold, and C.~Rother.
\newblock Learning 6d object pose estimation using 3d object coordinates.
\newblock In {\em Proceedings of the 14th European Conference on Computer
  Vision}, ECCV '14, 2014.

\bibitem{Drost:2010}
B.~Drost, M.~Ulrich, N.~Navab, and S.~Ilic.
\newblock Model globally, match locally: Efficient and robust 3d object
  recognition.
\newblock In {\em CVPR}, pages 998--1005. IEEE Computer Society, 2010.

\bibitem{Fischler:1981}
M.~Fischler and R.~Bolles.
\newblock Random sample consensus: A paradigm for model fitting with
  applications to image analysis and automated cartography.
\newblock {\em Communications of the ACM}, 24(6):381--395, 1981.

\bibitem{Gall:2011}
J.~Gall, A.~Yao, N.~Razavi, L.~J.~V. Gool, and V.~S. Lempitsky.
\newblock Hough forests for object detection, tracking, and action recognition.
\newblock {\em IEEE Trans. Pattern Anal. Mach. Intell.}, 33(11):2188--2202,
  2011.

\bibitem{Gordon:2006}
I.~Gordon and D.~G. Lowe.
\newblock {\em What and Where: 3D Object Recognition with Accurate Pose}.
\newblock Springer Berlin Heidelberg, Berlin, Heidelberg, 2006.

\bibitem{Hinterstoisser:2012:PAMI}
S.~Hinterstoisser, C.~Cagniart, S.~Ilic, P.~F. Sturm, N.~Navab, P.~Fua, and
  V.~Lepetit.
\newblock Gradient response maps for real-time detection of textureless
  objects.
\newblock {\em IEEE Trans. Pattern Anal. Mach. Intell.}, 34(5):876--888, 2012.

\bibitem{Hinterstoisser:2012}
S.~Hinterstoisser, V.~Lepetit, S.~Ilic, S.~Holzer, G.~R. Bradski, K.~Konolige,
  and N.~Navab.
\newblock Model based training, detection and pose estimation of texture-less
  3d objects in heavily cluttered scenes.
\newblock In {\em ACCV (1)}, pages 548--562, 2012.

\bibitem{Hinterstoisser:2016}
S.~Hinterstoisser, V.~Lepetit, N.~Rajkumar, and K.~Konolige.
\newblock Going further with point pair features.
\newblock In {\em Proceedings of the 15th European Conference on Computer
  Vision}, ECCV '16, 2016.

\bibitem{Hoiem:2007}
D.~Hoiem, C.~Rother, and J.~M. Winn.
\newblock 3d layoutcrf for multi-view object class recognition and
  segmentation.
\newblock In {\em CVPR}. IEEE Computer Society, 2007.

\bibitem{Huttenlocher:1993}
D.~P. Huttenlocher, G.~A. Klanderman, and W.~Rucklidge.
\newblock Comparing images using the hausdorff distance.
\newblock {\em IEEE Trans. Pattern Anal. Mach. Intell.}, 15(9):850--863, 1993.

\bibitem{Kabsch:1976}
W.~Kabsch.
\newblock {A solution for the best rotation to relate two sets of vectors}.
\newblock {\em Acta Crystallographica Section A}, 32(5):922--923, Sep 1976.

\bibitem{kappes-2015-ijcv}
J.~H. Kappes, B.~Andres, F.~A. Hamprecht, C.~Schn\"orr, S.~Nowozin, D.~Batra,
  S.~Kim, B.~X. Kausler, T.~Kr\"oger, J.~Lellmann, N.~Komodakis,
  B.~Savchynskyy, and C.~Rother.
\newblock A comparative study of modern inference techniques for structured
  discrete energy minimization problems.
\newblock {\em International Journal of Computer Vision}, pages 1--30, 2015.

\bibitem{Kolmogorov:Pami2006}
V.~Kolmogorov.
\newblock Convergent tree-reweighted message passing for energy minimization.
\newblock In {\em Transactions on Pattern Analysis and Machine Intelligence
  (PAMI)}, PAMI '06, 2006.

\bibitem{kolmogorov2007minimizing}
V.~Kolmogorov and C.~Rother.
\newblock Minimizing non-submodular functions with graph cuts-a review.
\newblock {\em IEEE transactions on pattern analysis and machine intelligence},
  29(7):1274--1279, 2007.

\bibitem{koltun2011efficient}
P.~Kr\"ahenb\"uhl and V.~Koltun.
\newblock Efficient inference in fully connected crfs with gaussian edge
  potentials.
\newblock {\em NIPS}, 2011.

\bibitem{Krull:2015}
A.~Krull, E.~Brachmann, F.~Michel, M.~Y. Yang, S.~Gumhold, and C.~Rother.
\newblock Learning analysis-by-synthesis for 6d pose estimation in rgb-d
  images.
\newblock In {\em In Proceedings of the 15th International Conference on
  Computer Vision}, ICCV '15, 2015.

\bibitem{Martinez:2010}
M.~Martinez, A.~Collet, and S.~S. Srinivasa.
\newblock Moped: A scalable and low latency object recognition and pose
  estimation system.
\newblock In {\em ICRA}, pages 2043--2049. IEEE, 2010.

\bibitem{Phillips:2016}
C.~J. Phillips, M.~Lecce, and K.~Daniilidis.
\newblock Seeing glassware: from edge detection to pose estimation and shape
  recovery.
\newblock In D.~Hsu, N.~M. Amato, S.~Berman, and S.~A. Jacobs, editors, {\em
  Robotics: Science and Systems}, 2016.

\bibitem{shekhovtsov-14}
A.~Shekhovtsov.
\newblock Maximum persistency in energy minimization.
\newblock In {\em CVPR}, pages 1162--1169. IEEE Computer Society, 2014.

\bibitem{Shekhovtsov:2015}
A.~Shekhovtsov, P.~Swoboda, and B.~Savchynskyy.
\newblock Maximum persistency via iterative relaxed inference with graphical
  models.
\newblock In {\em The IEEE Conference on Computer Vision and Pattern
  Recognition (CVPR)}, CVPR '15, 2015.

\bibitem{Shotton:2013:SCORF}
J.~Shotton, B.~Glocker, C.~Zach, S.~Izadi, A.~Criminisi, and A.~Fitzgibbon.
\newblock Scene coordinate regression forests for camera relocalization in
  rgb-d images.
\newblock In {\em Proc. Computer Vision and Pattern Recognition (CVPR)}. IEEE,
  June 2013.

\bibitem{Sun:2010}
M.~Sun, G.~R. Bradski, B.-X. Xu, and S.~Savarese.
\newblock Depth-encoded hough voting for joint object detection and shape
  recovery.
\newblock In K.~Daniilidis, P.~Maragos, and N.~Paragios, editors, {\em ECCV
  (5)}, volume 6315 of {\em Lecture Notes in Computer Science}, pages 658--671.
  Springer, 2010.

\bibitem{Swoboda:2014}
P.~Swoboda, B.~Savchynskyy, J.~H. Kappes, and C.~Schnörr.
\newblock Partial optimality by pruning for map-inference with general
  graphical models.
\newblock In {\em CVPR}, pages 1170--1177. IEEE Computer Society, 2014.

\bibitem{Swoboda2016}
P.~Swoboda, A.~Shekhovtsov, J.~Kappes, C.~Schn\"{o}rr, and B.~Savchynskyy.
\newblock {P}artial {O}ptimality by {P}runing for {MAP}-{I}nference with
  {G}eneral {G}raphical {M}odels.
\newblock {\em IEEE Trans.~Patt.~Anal.~Mach.~Intell.}, 38(7):1370--1382, 7
  2016.
\newblock http://doi.ieeecomputersociety.org/10.1109/TPAMI.2015.2484327.

\bibitem{Tejani:2014}
A.~Tejani, D.~Tang, R.~Kouskouridas, and T.-K. Kim.
\newblock Latent-class hough forests for 3d object detection and pose
  estimation.
\newblock In D.~J. Fleet, T.~Pajdla, B.~Schiele, and T.~Tuytelaars, editors,
  {\em ECCV (6)}, volume 8694 of {\em Lecture Notes in Computer Science}, pages
  462--477. Springer, 2014.

\bibitem{wang2016relaxation}
C.~Wang and R.~Zabih.
\newblock Relaxation-based preprocessing techniques for markov random field
  inference.
\newblock In {\em Proceedings of the IEEE Conference on Computer Vision and
  Pattern Recognition}, pages 5830--5838, 2016.

\bibitem{Winn:2006}
J.~Winn and J.~Shotton.
\newblock The layout consistent random field for recognizing and segmenting
  partially occluded objects.
\newblock In {\em Proceedings of IEEE CVPR}, January 2006.

\bibitem{Zach:2015}
C.~Zach, A.~Penate-Sanchez, and M.-T. Pham.
\newblock A dynamic programming approach for fast and and robust object pose
  recognition from range images.
\newblock In {\em The IEEE Conference on Computer Vision and Pattern
  Recognition (CVPR)}, CVPR '15, 2015.

\end{thebibliography}
